\documentclass[]{article} 
\usepackage{geometry}
\usepackage{url}
\usepackage[utf8]{inputenc} 
\usepackage[T1]{fontenc}    
\usepackage{lmodern}
\usepackage{hyperref}       
\usepackage{url}            
\usepackage{booktabs}       
\usepackage{amsfonts}       
\usepackage{nicefrac}       
\usepackage{microtype}      
\usepackage{xcolor}         
\usepackage{url}
\usepackage{booktabs}       
\usepackage{nicefrac}       
\usepackage{microtype}      
\usepackage{setspace}
\usepackage{makecell}
\usepackage{amsmath,amsthm,amsfonts,amssymb}
\usepackage{bm}
\usepackage{appendix}
\usepackage{comment}
\usepackage{tablefootnote}
\usepackage{multirow}
\usepackage{hhline}
\usepackage{graphicx}
\usepackage{subfigure}
\usepackage{algorithm,algorithmic}
\usepackage[font=footnotesize,labelfont=bf]{caption}
\usepackage{url}
\usepackage{hyperref}
\usepackage{cases}
\usepackage{enumitem}
\newtheorem{definition}{Definition}
\newtheorem{theorem}{Theorem}
\newtheorem{remark}{Remark}
\newtheorem{lemma}{Lemma}

\newtheorem{proposition}{Proposition}
\title{Multi-Class Classification from Single-Class Data with Confidences}
\author{
Yuzhou Cao$^{1}$, Lei Feng$^{2}$, Senlin Shu$^{3}$, Yitian Xu$^{1}$, Bo An$^4$, Gang Niu$^5$ , Masashi Sugiyama$^{5,6}$\\
  $^1$College of Science, China Agricultural University, China\\
  $^2$College of Computer Science, Chongqing University, China\\
  $^3$College of Computer and Information Science, Southwest University, China\\
  $^4$School of Computer Science and Engineering, Nanyang Technological University, Singapore\\
  $^5$RIKEN Center for Advanced Intelligence Project, Tokyo, Japan\\
  $^6$The University of Tokyo, Tokyo, Japan\\
  \texttt{nanjing.caoyuzhou@gmail.com}, \texttt{lfeng@cqu.edu.cn},\texttt{ssl2108@email.swu.edu.cn}\\
  \texttt{xytshuxue@126.com}, \texttt{boan@ntu.edu.sg}\\
  \texttt{gang.niu@riken.jp}, \texttt{sugi@k.u-tokyo.ac.jp}  
}
\date{}
\begin{document}

\maketitle

\begin{abstract}
Can we learn a multi-class classifier from only \emph{data of a single class}?
We show that without any assumptions on the loss functions, models, and optimizers, we can successfully learn a multi-class classifier from only data of a single class with a rigorous consistency guarantee when \emph{confidences} (i.e., the class-posterior probabilities for all the classes) are available. 
Specifically, we propose an empirical risk minimization framework that is loss-/model-/optimizer-independent.
Instead of constructing a boundary between the given class and other classes, our method can conduct discriminative classification between all the classes even if no data from the other classes are provided.
We further theoretically and experimentally show that our method can be Bayes-consistent with a simple modification even if the provided confidences are highly noisy. Then, we provide an extension of our method for the case where data from a subset of all the classes are available. Experimental results demonstrate the effectiveness of our methods.
\end{abstract}
\section{Introduction}
In supervised learning, the annotations of a huge number of instances may not be easily obtained in many practical applications due to the concerns including but not limited to time consumption, expenditure, and privacy preserving. For these reasons, many \textit{weakly supervised learning} (WSL) frameworks \cite{Zhou2018A} have been studied in various scenarios recently, including \textit{semi-supervised learning} \cite{Chapelle, Xiaojin_Zhu, Gang_Niu, S4VM, PNU, Yufeng_Li, Lanzhe_Guo}, \textit{positive-unlabeled learning} \cite{Elkan, PUa, PUb, EPU}, \textit{unlabeled-unlabeled learning} \cite{UU, UUNN}, \textit{noisy-label learning} \cite{Noise_a, Masking, CT, LoIS}, \textit{complementary} and \textit{partial-label learning} \cite{Comp_a, BComp, Best_Comp, MCL, Limited_MCL, UGE, Partial_a, Consistent_Partial, Deterministic_Partial}, \textit{similarity-based learning} \cite{SU, SD, Sconf}, and \textit{positive-confidence learning} \cite{Pconf}.

In this paper, we investigate a novel weakly supervised learning scenario called \textbf{\textit{Single-Class Confidence (SC-Conf) classification}}, where only data of a single class annotated with the \textit{confidences} (i.e., the class-posterior probabilities for all the classes) are required for training a {\textit{discriminative}} multi-class classifier with \textit{convergence guarantee on estimation error}, without any data from other classes. Such a WSL framework can be widespread in many real-world scenarios. For example, in the scenario of climatic disaster forecasting, in order to conduct ordinary supervised classification, we deploy sensors to collect data of normal climate and different kinds of meteorological disasters. However, under extreme climate conditions, e.g., typhoons, tornados, and snowstorms, the sensors may not work properly and we can only collect data from the distribution of normal climate. To construct a climatic disaster forecasting system that can discriminate different kinds of climates with only data of normal climate, we can ask meteorologists to annotate the data with a confidence score and train the system using our proposed framework with only data of normal climate. Another example is about market investigation. The investigators of each company aim to predict the consumers' tendency of purchasing the products of their companies and other competitors. However, due to privacy concerns, there may be data isolation, i.e., each company can only have access to their customers' information. To accomplish the investigation with the isolated data, the investigators can gather the purchase history of their customers and transform the purchase amount of each company into confidence between 0 and 1 by pre-processing. With the isolated data and confidence scores, the predictor can be trained by SC-Conf classification.

To apply the SC-Conf classification framework in more realistic scenarios, we show that with a small modification, the result of our method with extremely noisy confidence can be consistent with that with accurate confidence, i.e., our framework is noise-robust. To justify our claims, both infinite-sample and finite-sample analysis are provided.

Furthermore, we extend our method to the case where we can collect data from a subset of all the classes with confidences. For example, when conducting multi-class classification of climatic disasters including sandstorm, snowstorm, drought, and flood, we can train the classifier using data only with a coarse label `wind storm' instead of collecting data from all four classes. Here we don't have to specify the true labels of each instance, e.g., the further division of data with coarse label `wind storm' into `sandstorm' and `snowstorm' is unnecessary. We refer to this framework as \textbf{\textit{Subset Confidence (Sub-Conf) classification}} in the rest part of this paper.

Our contributions in this paper are three-fold:
\begin{itemize}[leftmargin=0.4cm,topsep=-2pt]
\item We provide an unbiased estimator of ordinary classification risk with only data of a single class and their confidences. An optimizer-independent empirical risk minimization (ERM) framework with no assumptions on loss functions and models is proposed. We further establish the estimation error bound to show the consistency of the proposed method.
\item We show that if unlabeled samples are available, our proposed method is noise-robust and can be \textit{classifier-consistent} and can converge to the optimal classifier with high probability even with extremely noisy confidence. Furthermore, we give a novel finite-sample convergence analysis on the misclassification rate of this method. 
\item An extension of our ERM framework to the case where data from a subset of all the classes with confidences are provided.
\end{itemize}
Extensive experiments with deep neural networks on both benchmark and real-world datasets are conducted for demonstrating the usefulness of our proposed methods.
\section{Related Work}
In this section, we introduce related studies of the proposed SC-Conf classification framework. We briefly illustrate the proposed problem and related problems in Figure \ref{F1}.
\paragraph{Multi-Class Classification}
\begin{figure}
\centering
\includegraphics[height=3.6cm,width=16cm, trim={80 290 140 70},clip]{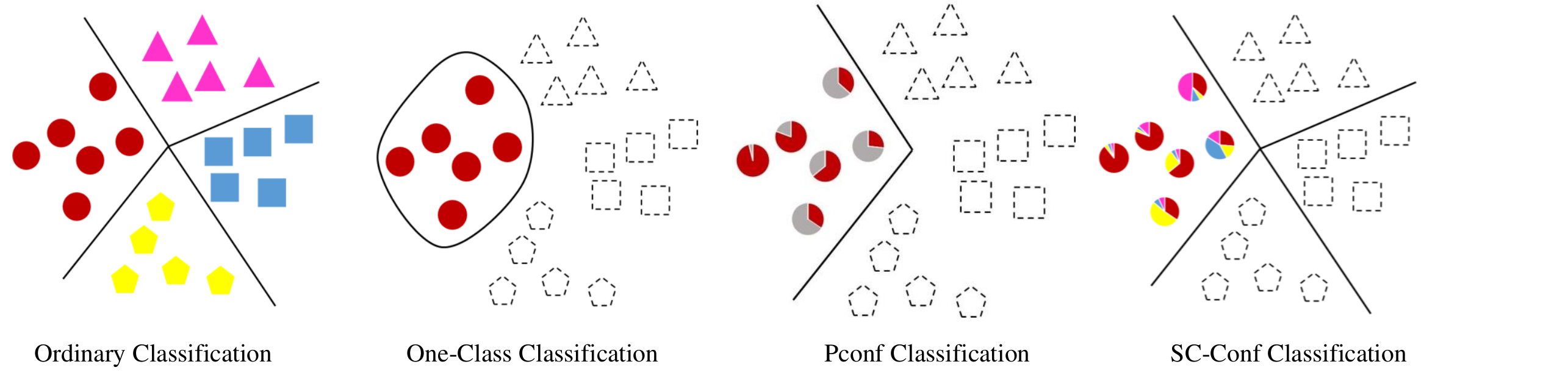}
\includegraphics[width=12cm, trim={2 360 190 100},clip]{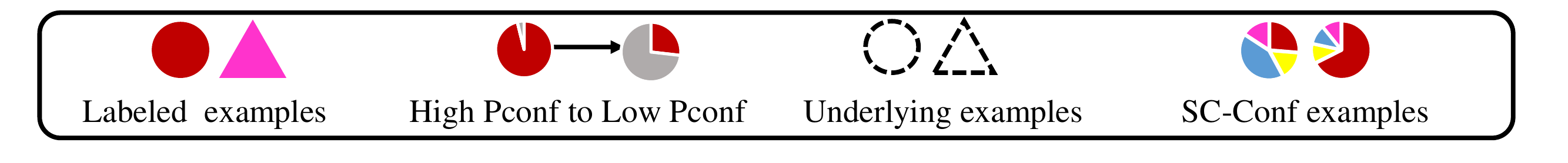}
\caption{Illustration of the proposed method and related works. The points with dotted lines are not necessary in the training process and are only shown for illustration. We further show in Section \ref{S4} that even we have no information about the value of confidence scores, we can still conduct consistent SC-Conf classification only with the class with the maximum class-posterior probability: $y'={\rm argmax}_{y\in\mathcal{Y}_{s}}p(y|\bm{x})$. We also show in Section \ref{S5} that it is possible to learn from a subset of all the classes.}
\label{F1}
\vspace{-7pt}
\end{figure}
\vspace{-7pt}
For ordinary $K$-class multi-class classification problem, $\mathcal{X}\subset\mathbb{R}^{d}$ is the feature space and $\mathcal{Y}=[K]$ is the label space. Suppose each example $(\bm{x},y)\in\mathcal{X}\times\mathcal{Y}$ is drawn independently from an unknown distribution with density $p(\bm{x},y)$. To train a multi-class classifier $\bm{g}:\mathcal{X}\rightarrow\mathbb{R}^{K}$, we have to minimize the classification risk below:
\begin{align}
\label{OR}
\textstyle R(\bm{g})=\mathbb{E}_{p(\bm{x},y)}\left[\ell(\bm{g}(\bm{x}),y)\right],
\end{align}
where $\ell(\cdot,\cdot):\mathbb{R}^{K}\times\mathcal{Y}\rightarrow\mathbb{R}_{+}$ is the multi-class loss function and $\mathbb{E}_{p(\bm{x},y)}[\cdot]$ is the expectation on distribution with density $p(\bm{x},y)$.  The predicted label of $\bm{x}$ is given as
$f(\bm{x})=\mathop{\rm argmax}\nolimits_{y\in\mathcal{Y}}g_{y}(\bm{x})$,
where $g_{y}(\bm{x})$ is the $y$-th element of $\bm{g}(\bm{x})$. Since we are not aware of the joint density $p(\bm{x},y)$, we collect identically and independently distributed examples $\{(\bm{x}_{i},y_{i})\}_{i=1}^{n}$ and minimize the empirical risk (the sample mean of the losses of the collected examples) instead.
\vspace{-7pt}
\paragraph{Multi-Positive and Unlabeled Learning.}
Multi-positive and unlabeled learning \cite{MPU} is a WSL framework that can train multi-class classifiers using labeled data from $K-1$ classes, unlabeled data collected from the distribution with density $p(\bm{x})$, and the class-prior probability of the unseen class. However, it is quite hard to cope with these requirements in real-world scenarios. On the other hand, our proposed SC-Conf classification framework only needs data of a single class and their confidence while it can still train multi-class classifiers.
\vspace{-7pt}
\paragraph{One-Class Classification.} One-class classification \cite{OC1, SVDD, OCSVM, DeepOC} aims at `describing' a given class with data of a single class rather than conducting discriminative classification. Furthermore, it cannot construct classification boundaries between all of the $K$ classes with a finite-sample convergence guarantee since it has no access to the information of the rest \textit{K-1} classes. Compared with one-class classification, our proposed SC-Conf classification framework can construct discriminative multi-class classifiers with only data of a single class by utilizing the confidence scores. Due to the discussion above, it can be seen that the one-class classification methods are not suitable in the problem of multi-class classification.
\vspace{-7pt}
\paragraph{Positive-Confidence Classification.} In \cite{Pconf, Pconf2}, Pconf classification is proposed to train discriminative binary classifiers only from positive examples and their positive-confidence (class-posterior probabilities of positive class). It can be regarded as a special case of our proposed SC-Conf classification when the number of classes is $K=2$ since the class-posterior probability of the negative class can be immediately obtained when the Pconf is given. Nevertheless, it is still confined in the binary classification setting. When applied to multi-class classification, the Pconf framework can only discriminate the class regarded as the `positive' class from the rest of all the classes, while the SC-Conf can conduct discriminative classification between all the classes, as shown in Figure \ref{F1}.
\section{SC-Conf Classification}
\label{S3}
In this section, we formulate the SC-Conf classification problem. In ordinary multi-class classification problem, we collect \textit{i.i.d.} labeled data $\{(\bm{x}_{i},y_{i})\}_{i=1}^{n}$ from the joint distribution $p(\bm{x},y)$ and conduct ERM \cite{SLT} by minimizing the unbiased risk estimator of classification risk (\ref{OR}) constructed by the labeled data. However, under the SC-Conf classification framework, we are only given data from a \textbf{single} class $y_{\mathrm{s}}$ with confidence: $\mathcal{S}=\{\bm{x}_{i},\bm{r}_{i}\}_{i=1}^{n}$, where $\bm{x}_{i}$ is an instance drawn independently from the distribution of a single class with density $p(\bm{x}|y=y_{\mathrm{s}})$ and $\bm{r}_{i}$ is a confidence vector given by $\{p(y|\bm{x}_{i})\}_{y=1}^{K}$\footnote{Note that in real-world applications, the equation $\bm{r}_{i}=p(y|\bm{x}_{i})$ may not hold due to the noisy supervision. We analyze the effect of noisy confidence in Section \ref{S4}.}. Based on these information, we show how to construct an unbiased risk estimator and conduct ERM without any instance from the other $K-1$ classes in this section.
\subsection{Unbiased Risk Estimator}
We denote the class-prior probability of class $y_{\mathrm{s}}$ as $\pi_{y_s}$. Let $\bm{r}(\bm{x})=\{p(y|\bm{x})\}_{y=1}^{K}$ and $r^{i}(\bm{x})=p(y=i|\bm{x})$. $\mathbb{E}_{p(\bm{x}|y_{\mathrm{s}})}[\cdot]$ is the expectation on the distribution with density $p(\bm{x}|y_{\mathrm{s}})$. The following theorem shows that the classification risk can be recovered with the supervision above:
\begin{theorem}
\label{URE-SC}
The multi-class classification risk can be recovered from data drawn from a single class if their confidence scores are given:
\begin{align}
\label{SC}
\textstyle
R(\bm{g})=\pi_{y_{\mathrm{s}}}\mathbb{E}_{p(\bm{x}|y_{\mathrm{s}})}\left[\sum_{y=1}^{K}\frac{r^{y}(\bm{x})}{r^{y_{\mathrm{s}}}(\bm{x})}\ell(\bm{g}(\bm{x}),y)\right],
\end{align}
where $p(y_{\mathrm{s}}|\bm{x})>0$ for all $\bm{x}$ in the support of the distribution with density $p(\bm{x})$.
\end{theorem}
The proof is shown in Appendix A. A sketch of the proof is that (\ref{SC}) is essentially duplicating $(\bm{x},\bm{r})$ into $\left(\bm{x},p(y=1|\bm{x})\right), \left(\bm{x},p(y=2|\bm{x})\right),...,\left(\bm{x},p(y=K|\bm{x})\right)$ and then applying importance weighting to adjust the difference between the distribution of a single class and all the classes. According to Theorem \ref{URE-SC}, we can see that Eq.~(\ref{SC}) is an equivalent risk expression of Eq.~(\ref{OR}). Then we can directly get the unbiased risk estimator of Eq.~(\ref{SC}):
\begin{align}
\textstyle
\label{USC}
\hat{R}_{\mathrm{SC}}(\bm{g})=\frac{\pi_{y_{\mathrm{s}}}}{n}\sum\nolimits_{i=1}^{n}\sum_{y=1}^{K}\frac{r^{y}_{i}}{r^{y_{\mathrm{s}}}_{i}}\ell\left(\bm{g}(\bm{x_{i}}),y\right).
\end{align}
Since Eq.~(\ref{SC}) does not include the expectation over the data of the other $K-1$ classes, only data from the single class $y_{\mathrm{s}}$ are used in Eq.~(\ref{USC}) for constructing the unbiased risk estimator. Then the following work is to conduct ERM by minimizing the unbiased risk estimator in Eq.~(\ref{USC}). Since the class-prior probability $\pi_{y_{\mathrm{s}}}$ can be safely ignored without changing the result of ERM, we do not have to estimate $\pi_{s}$ in the process of SC-Conf classification.
\subsection{Estimation Error Bound}
\label{sbound}
Here we analyze the consistency of the proposed unbiased risk estimator $\hat{R}_{SC}(\bm{g})$ in (\ref{USC}) and give the convergence rate of its estimation error. First of all, let $\bm{\mathcal{G}}=[\mathcal{G}_{y}]_{y=1}^{K}$ be a class of $K$-valued functions that we consider in ERM. Assume that there is $C_{\bm{g}}>0$ that ${\rm sup}_{\bm{g}\in \bm{\mathcal{G}}}\|\bm{g}\|_{\infty}\leq C_{\bm{g}}$. The loss function $\ell(\bm{g}(\bm{x}),y)$ is Lipschitz continuous for all $\|\bm{g}\|_{\infty}\leq C_{\bm{g}}$ with Lipschitz constant $L_{\ell}>0$ and upper-bounded by $C_{\ell}>0$. At the end, we assume that the confidence of the given single class is not too small that there exists $C_{r}>0$ such that $p(y_{\mathrm{s}}|\bm{x})>C_{r}$ holds almost surely. Though we can omit the $p(y_{\mathrm{s}}|\bm{x})$ and analyze the biased version of $\hat{R}_{SC}(\bm{g})$, for simplicity, only the unbiased version is analyzed here.

Denote by $\hat{\bm{g}}_{SC}$ the minimizer of $\hat{R}_{SC}(\bm{g})$ in (\ref{USC}) and $\bm{g}^{*}$ the minimizer of (\ref{OR}). Then the estimation error is given by $R(\hat{\bm{g}})-R(\bm{g}^{*})$. The following theorem gives the estimation error bound:
\begin{theorem}\label{Bound} For any $0<\delta<1$, with probability at least 1-$\delta$:
\begin{align}
\textstyle
R(\hat{\bm{g}}_{SC})-R(\bm{g}^{*})\leq\frac{4\sqrt{2}\pi_{y_{\mathrm{s}}}L_{\ell}}{C_{r}}\sum_{y=1}^{K}\mathfrak{R}_{n}(\mathcal{G}_{y})+2\pi_{y_{\mathrm{s}}}C_{\ell}\sqrt{\frac{\log\frac{2}{\delta}}{2n}},
\end{align}
where $\mathfrak{R}_{n}(\mathcal{G}_{y})$ is the Rademacher complexity \cite{rade} of $\mathcal{G}_{y}$ on \textit{i.i.d.} samples with size $n$ drawn from a single class $p(\bm{x}|y_{\mathrm{s}})$.
\end{theorem}
The definition of Rademacher complexity and the proof is given in the Appendix B. Generally, the Rademacher complexity $\mathfrak{R}_{n}(\mathcal{G}_{y})$ has an upper bound of $C_{\mathcal{G}_{y}}/\sqrt{n}$ \cite{RadeNN}, where $C_{\mathcal{G}_{y}}$ is a positive number determined by the function class $\mathcal{G}_{y}$. We can see that the risk of empirically optimal classifier $R(\hat{\bm{g}}_{SC})$ converges to that of the optimal classifier $R(\bm{g}^{*})$ in the rate of $O_{p}(1/\sqrt{n})$, which is the optimal parametric rate in probability and cannot be improved without additional assumptions \cite{Opt}.
\section{Consistent Noise-Robust Method for SC-Conf Classification}
\label{S4}
In Section \ref{S3}, it can be seen that the confidences play an important role in the construction of the unbiased risk estimator (\ref{USC}). However, in practical applications, the confidences $\{\bm{r}_{i}\}_{i=1}^{n}$ are often inaccurate, and we can only get the noisy confidence scores $\{\tilde{\bm{r}}_{i}\}_{i=1}^{n}$ generated by a corrupted distribution. Since $\tilde{\bm{r}}_{i}\not=\bm{r}_{i}$, the risk estimator proposed in Section \ref{S3} is biased in general if we simply plug the noisy confidence into (\ref{USC}). As a result, the learning guarantee in the previous section no longer holds and thus the use of ERM is not reasonable.

Can we obtain a provably consistent multi-class classifier with data of a single class and extremely noisy confidence scores? Surprisingly, in this section, we give a positive answer to this question. 

We propose a refined method called \textbf{Noise-Robust Single-Class Confidence} (NoRSC-Conf) classification, where only data of a single class, unlabeled data, and noisy confidence scores are provided. Then we show that this method is \textit{classifier-consistent} \cite{BComp, Consistent_Partial, Deterministic_Partial}: the predictions of the classifier generated from noisy supervision is still infinite-sample consistent to those of the optimal classifier of classification risk $R(\bm{g})$. Based on this result, we establish the ERM framework for NoRSC-Conf classification. Since there is a gap between the convergence on the corrupted distribution and the original one, we cannot directly analyze the finite-sample consistency of NoRSC-Conf classification with traditional technology on bounding the uniform convergence. To justify the use of ERM, we exploit the connection between the noisy and original distribution and give a novel finite-sample analysis on the misclassification rate of NoRSC-Conf classification. Finally, we show that the weight $\phi(\bm{x})$ used in the construction of NoRSC-Conf classification can be estimated from unlabeled data and data of a single class by applying the density ratio matching-based method \cite{Density} and give its estimation error bound.
\subsection{Formulation of NoRSC-Conf Classification and Infinite-Sample Consistency}
In the setting of NoRSC-Conf classification, we are provided with \textit{i.i.d.} unlabeled data $\mathcal{S}_{u}=\{\bm{x}^{u}_{i}\}_{i=1}^{n_{u}}$ drawn from the distribution with marginal density $p(\bm{x})$, data of a single class with noisy confidence scores $\widetilde{\mathcal{S}}=\{\bm{x}_{i},\tilde{r}_{i}\}_{i=1}^{n}$. Here the generation of data of a single class $\{\bm{x}_{i}\}_{i=1}^{n}$ is the same with that in the previous section, and the noisy confidence $\{\tilde{\bm{r}}\}_{i=1}^{n}$ are generated by a corrupted distribution with density $\tilde{p}(y|\bm{x})$. 

Suppose the optimal classifier $\tilde{\bm{g}}^{\mathrm{bayes}}=\mathop{\rm argmin}_{\bm{g}~{\rm measurable}}\tilde{R}(\bm{g})$ is in the function class $\bm{\mathcal{G}}$, which is a widely used assumption in classifier-consistency analysis, e.g., the Assumption 1 in \cite{BComp}. Straightforwardly, we give the the formulation and show the infinite-sample consistency of the NoRSC-Conf classification method:



\begin{theorem}
\label{T2}
For each instance $\bm{x}$, if $\mathop{\rm argmax}_{y\in\mathcal{Y}}\tilde{p}(y|\bm{x})=\mathop{\rm argmax}_{y\in\mathcal{Y}}p(y|\bm{x})$ and the classification-calibrated losses (e.g., softmax cross entropy loss, mean squared error) \cite{MCC} are used, the following reweighed risk formulation is classifier-consistent:
\begin{align}
\textstyle
\label{NSC}
\tilde{R}(\bm{g})=\mathbb{E}_{p(\bm{x}|y_{\mathrm{s}})}\left[\phi(\bm{x})\sum_{y=1}^{K}\tilde{r}^{y}(\bm{x})\ell(\bm{g}(\bm{x}),y)\right],
\end{align}
where $\phi(\bm{x})=\frac{p(\bm{x})}{p(\bm{x}|y_{\mathrm{s}})}$,  $\tilde{r}^{y}(\bm{x})=\tilde{p}(y|\bm{x})$, and $p(y_{\mathrm{s}}|\bm{x})>0$ for all $\bm{x}$ in the support of the distribution with density $p(\bm{x})$.
\end{theorem}
An intuitive explanation is that the decision boundary of the Bayes-optimal classifier is only determined by the class with the maximum class-posterior probability of each instance. Then we only have to adjust the distribution difference between the distribution of a single class and all the classes. The detailed proof is shown in Appendix C. With regard to the assumption on the noisy confidence, Theorem \ref{T2} only requires that for each instance $\bm{x}$, the class with the largest posterior probability in the original distribution should still have the largest value among all the classes in the noisy distribution. Notice that since there is no assumption on the exact value of the confidence scores, the confidence scores can be of arbitrary values in the $K$-dimensional probability simplex as long as the assumption above is satisfied. This assumption is realistic in practice. For example, when asking the annotators for labeling the data of a single class with confidence scores, they may not be sure about the exact value of the true class-posterior probability of each class. However, it can be easier for them to identify the class of ${\rm argmax}_{y\in\mathcal{Y}} p(y|\bm{x})$ and give the highest but not necessarily the accurate confidence score to this class. Though this kind of confidence score is not accurate in general, according to Theorem \ref{T2}, it is meaningful in the setting of NoRSC-Conf classification and can be used for generating consistent classifiers.

Denote by $\tilde{\bm{g}}^{*}$ the minimizer of (\ref{NSC}), then Theorem \ref{T2} means that $\tilde{\bm{g}}^{*}$ and $\bm{g}^{*}$ have the same prediction, i.e., ${\rm argmax}_{y\in\mathcal{Y}}\tilde{g}^{*}_{y}(\bm{x})={\rm argmax}_{y\in\mathcal{Y}}g^{*}_{y}(\bm{x})$ for all the $\bm{x}$. Thus the prediction of the minimizer of (\ref{NSC}) is infinite-sample consistent to that of the ordinary classification risk (\ref{OR}). Based on this result and $\phi_{i}=\phi(\bm{x}_{i})$, it is common practice to conduct ERM by minimizing the unbiased estimator of (\ref{NSC}): 
\begin{align}
\textstyle
\label{ERM-SC}
\hat{\tilde{R}}(\bm{g})=\frac{1}{n}\sum_{i=1}^{n}\phi_{i}\sum_{y=1}^{K}\tilde{r}^{y}_{i}\ell(\bm{g}(\bm{x}_{i}),y).
\end{align}

Denote by $\hat{\tilde{\bm{g}}}$ the minimizer of the empirical risk (\ref{NSC}). Using the same technology of bounding the uniform convergence as in Section \ref{sbound}, we can get the finite-sample analysis of the result: $\tilde{R}(\hat{\tilde{\bm{g}}})\rightarrow\tilde{R}({\tilde{\bm{g}}}^{*})$. Nevertheless, in the scenario of multi-class classification, we are concerning about the problem that if the misclassification risk of $\hat{\tilde{\bm{g}}}$ on the original distribution with density $p(\bm{x},y)$ converges to the optimal one, i.e., if $R_{01}(\hat{\tilde{\bm{g}}})\rightarrow R_{01}(\bm{g}_{01}^{*})$\footnote{According to the property of classification-calibrated loss, the minimizer $\bm{g}^{*}$ of $R(\bm{g})$ with surrogate loss $\ell$ is also the minimizer of misclassification rate $R_{01}(\bm{g})$ among all the measurable functions, where $1(\cdot)$ is the indicator function.}, where $R_{01}(\bm{g})=\mathbb{E}_{p(\bm{x},y)}\left[1({\rm argmax}_{y'\in\mathcal{Y}}g_{y'}(\bm{x})\not=y)\right]$ is the equivalent expression of misclassification rate of classifier $\bm{g}$, and $\bm{g}^{*}_{01}={\rm argmin}_{\bm{g}\in\bm{\mathcal{G}}}R_{01}(\bm{g})$ is the optimal classifier of misclassification rate. Since there is difference between the noisy and original distribution, and we cannot get the convergence result of $R_{01}(\hat{\tilde{\bm{g}}})$ directly from the fact that $\tilde{R}(\hat{\tilde{\bm{g}}})\rightarrow\tilde{R}({\tilde{\bm{g}}}^{*})$. Thus the finite-sample analysis of the misclassification rate of $\hat{\tilde{\bm{g}}}$ is non-trivial.

In the following section, we will give the answer to this question by exploiting the connection between $\tilde{p}(y|\bm{x})$ and $p(y|\bm{x})$ and give a finite-sample analysis on the misclassification rate of $\hat{\tilde{\bm{g}}}$.

\subsection{Finite-Sample Analysis on the Misclassification Rate}
In this section, we give a non-trivial finite-sample analysis on the misclassification rate of $\hat{\tilde{\bm{g}}}$: $R_{01}(\hat{\tilde{\bm{g}}})-R_{01}(\bm{g}_{01}^{*})$. To begin with, we first give the convergence analysis on $\tilde{R}(\hat{\tilde{\bm{g}}})\rightarrow\tilde{R}({\tilde{\bm{g}}}^{*})$.
\begin{lemma}\label{lbound}\!
Based on the assumptions in previous parts, for any 0<$\delta$<1, with probability at least 1-$\delta$:
\begin{align}\textstyle
\tilde{R}(\hat{\tilde{\bm{g}}})-\tilde{R}(\tilde{\bm{g}}^{*})\leq\frac{4\sqrt{2}L_{\ell}}{C_{r}}\sum_{y=1}^{K}\mathfrak{R}_{n}(\mathcal{G}_{y})+\frac{2C_{\ell}}{C_{r}}\sqrt{\frac{\log\frac{2}{\delta}}{2n}},
\end{align}
where $\mathfrak{R}_{n}(\mathcal{G}_{y})$ is the Rademacher complexity of $\mathcal{G}_{y}$ on \textit{i.i.d.} samples with size $n$ drawn from a single class $p(\bm{x}|y_{\mathrm{s}})$.
\end{lemma}
Combining this lemma with the Corollary 26 in \cite{zhanga}, we can directly bound the following term \textit{w.r.t.} 0-1 loss: $\tilde{R}_{01}(\hat{\tilde{\bm{g}}})-\tilde{R}_{01}({\tilde{\bm{g}}}^{*})$:
\begin{lemma}\label{tool} If the loss function is classification-calibrated, there exists a concave function $\xi$ on $[0,+\infty)$ such that $\xi(0) = 0$ and $\xi(\delta)\rightarrow0$ as $\delta \rightarrow 0^{+}$. For any 0<$\delta$<1, with probability at least 1-$\delta$:
\begin{align}
\textstyle
\tilde{R}_{01}(\hat{\tilde{\bm{g}}})-\tilde{R}_{01}({\tilde{\bm{g}}}^{*})\leq\xi\left(\tilde{R}(\hat{\tilde{\bm{g}}})-\tilde{R}(\tilde{\bm{g}}^{*})\right).
\end{align}
where $\tilde{R}_{01}(\bm{g})$ is the misclassification rate of $\bm{g}$ on the noisy distribution with density $\tilde{p}(\bm{x},y)=\tilde{p}(y|\bm{x})p(\bm{x})$. The validity of the density $\tilde{p}(\bm{x},y)$ is shown in the proof of Theorem \ref{T2}.
\end{lemma}
The lemma above shows that $\tilde{R}_{01}(\hat{\tilde{\bm{g}}})\rightarrow\tilde{R}_{01}({\tilde{\bm{g}}}^{*})$ in probability as $n\rightarrow+\infty$. However, this conclusion is still not informative enough: since there is a gap between the noisy density $\tilde{p}(\bm{x},y)$ and the original distribution $p(\bm{x},y)$, the conclusion $\tilde{R}_{01}(\hat{\tilde{\bm{g}}})\rightarrow\tilde{R}_{01}({\tilde{\bm{g}}}^{*})$ in probability cannot directly give an answer to the question that if the risk of the empirical minimizer of NoRSC-Conf classification $R_{01}(\hat{\tilde{\bm{g}}})$ can converge to the optimal misclassification rate $R_{01}(\bm{g}_{01}^{*})$ on the original distribution in probability. To answer this question, we have to make use of the connection between $\tilde{p}(\bm{x},y)$ and $p(\bm{x},y)$. For each $\bm{x}\in\mathcal{X}$, let $\Delta(\bm{x})\!=\!\tilde{p}(y|\bm{x})\!-\!\tilde{p}(y'|\bm{x})$, where $y$ and $y'$ are the classes with the largest and the second largest posterior possibilities, respectively. Then we have the following conclusion: 
\begin{theorem} \label{final}
Suppose $\inf_{\bm{x}\in\mathcal{X}}\Delta(\bm{x})>0$, then for any $0<\delta<1$, with probability at least $1-\delta$: 
\begin{align}
\textstyle
R_{01}(\hat{\tilde{\bm{g}}})-R_{01}({\bm{g}}_{01}^{*})\leq\frac{1}{\inf_{\bm{x}\in\mathcal{X}}\Delta(\bm{x})}&\xi\left({\tilde{R}(\hat{\tilde{\bm{g}}})-\tilde{R}(\tilde{\bm{g}}^{*})}\right)
\end{align}
\end{theorem}
We prove this conclusion in Appendix D. Combining the Lemma \ref{lbound}, Theorem \ref{final}, and the property of function $\xi$, we can finally say that with the increasing of sample size $n$, the misclassification rate of $\hat{\tilde{\bm{g}}}$ converges to the optimal misclassification rate $R_{01}(\bm{g}_{01}^{*})$ in probability. According to this conclusion, we can see that the minimizer of the empirical risk \ref{ERM-SC} possesses finite-sample consistency, which justifies the use of ERM.
\begin{remark}
\label{r1}
\rm A representative scenario is that only the class with maximum class-posterior probability is identified for each data point. Then the noisy confidences are given in the form of one-hot codes: $\tilde{\bm{r}}(\bm{x})=[0,\cdots,1,\cdots,0]$, where the only non-zero element 1 is on the $y_{sc}$-th position and $y_{sc}$ is the class with maximum class-posterior probability. In this scenario, almost all the details of the confidence of each class are lost and the given confidences are extremely noisy. It can be seen that in this scenario, $\frac{1}{\inf_{\bm{x}\in\mathcal{X}}\Delta(\bm{x})}=1$. According to Theorem \ref{final}, the result of our NoRSC-Conf classification can still converge to the Bayes-optimal classifier in probability. 
\end{remark}
\subsection{Efficient Estimation of Density Ratio $\phi$(\textit{\textbf{x}}) with Unlabeled Data}
Though the infinite-sample consistency and the finite-sample consistency of our NoRSC-Conf classification have been established, there is still an important problem to be solved: the estimation of the weight term $\phi(\bm{x})$. In this section, we show that additional information of the class labels is not required in the estimation of the weight term $\phi(\bm{x})$ and only the data from only a single class and unlabeled data are needed. 

Let $\phi(\bm{x})$ be the true density ratio $\frac{p(\bm{x})}{p(\bm{x}|y_{\mathrm{s}})}$ and $\hat{\phi}(\bm{x})$ be the estimated one. We can empirically approximate $\phi(\bm{x})$ by the density ratio matching method \cite{Density}. First, we give the definition of the Bregman divergence, which can measure the discrepancy between two density ratio models:
\begin{definition}\cite{Density}
For any differentiable and strictly convex function $\eta$: $\mathbb{R}\rightarrow\mathbb{R}$, $\nabla \eta(t)$ is the subgradient of $\eta$. The Bregman divergence of $\eta$ between the true density ratio $\phi(\bm{x})$ and the estimated density ratio $\hat{\phi}(\bm{x})$ is given as:
\begin{align*}\textstyle
B_{\eta}(\phi\|\hat{\phi})=\int p(\bm{x}|y_{\mathrm{s}})\nabla\eta(\hat{\phi}(\bm{x}))\hat{\phi}(\bm{x})d\bm{x}-\int p(\bm{x}|y_{\mathrm{s}})\eta(\hat{\phi}(\bm{x}))d\bm{x}-\int p(\bm{x})\nabla\eta(\hat{\phi}(\bm{x}))d\bm{x},
\end{align*}
and its unbiased estimator is given as:
\begin{align}
\textstyle
\label{eb}
\hat{B}_{\eta}(\phi\|\hat{\phi})=\frac{1}{n}\sum_{i=1}^{n}\nabla\eta(\hat{\phi}(\bm{x}_{i}))\hat{\phi}(\bm{x}_{i})-\frac{1}{n}\sum_{i=1}^{n}\eta(\hat{\phi}(\bm{x}_{i}))-\frac{1}{n_{u}}\sum_{i=1}^{n_{u}}\nabla\eta(\hat{\phi}(\bm{x}_{i}^{u})).    
\end{align}
\end{definition}

Since we are given data from the single class $y_{\mathrm{s}}$ and the unlabeled data, we can estimate the true density ratio by minimizing
the empirical Bregman divergence (\ref{eb}). Denote by $\Phi$ the function class of ratio model and $\hat{\phi}^{*}={\rm argmin}_{\hat{\phi}\in\Phi}\hat{B}_{\eta}(\phi\|\hat{\phi})$ and assume that the true ratio model $\phi\in\Phi$, we have the following estimation error bound:
\begin{theorem}
\label{final_bound}
Based on the assumptions above, for any $0<\delta<1$, with probability at least $1-\delta$:
\begin{align*}
\textstyle
B_{\eta}(\phi\|\hat{\phi}^{*})\leq C_{1}\mathfrak{R}_{n}(\Phi)+C_{2}\mathfrak{R}^{u}_{n_{u}}(\Phi)+M_{1}\sqrt{\frac{\ln \frac{4}{\delta}}{2n}}+M_{2}\sqrt{\frac{\ln \frac{4}{\delta}}{2n_{u}}}
\end{align*}
where $\mathfrak{R}_{n}(\Phi)$ is the Rademacher complexity of function class $\Phi$ on \textit{i.i.d.} samples with size $n$ drawn from a single class $p(\bm{x}|y_{\mathrm{s}})$, $\mathfrak{R}^{u}_{n_{u}}(\Phi)$ is the Rademacher complexity of function class $\Phi$ on \textit{i.i.d.} samples with size $n_{u}$ drawn from marginal distribution with density $p(\bm{x})$. Meanwhile, $C_{1},~C_{2},~M_{1},~M_{2}$ are constants.
\end{theorem}
The proof is given in Appendix E. This estimation error bound guarantees that the estimated ratio $\hat{\phi}$ converges to the true one in the rate of $\mathcal{O}_{p}(1\!/\!\sqrt{n_{u}}\!+\!1\!/\!\sqrt{n})$. In other words, our estimated ratio model will be more accurate as the numbers of data from the single class $y_{\mathrm{s}}$ and unlabeled data increases.
\section{Sub-Conf Classification}
\label{S5}
In many practical situations, the data may not be collected from only a single class $y_{\mathrm{s}}\in\mathcal{Y}$ but from a subset $\mathcal{Y}_{s}\subset\mathcal{Y}$ of all the classes. Here we give the data generation process of this kind of data. Suppose we have data drawn from a subset of all the classes and their confidence: $\{\bm{x}_{i},\bm{r}_{i}\}_{i=1}^{}$, where the confidence is the same with that in Section \ref{S3} and the data $\{\bm{x}_{i}\}_{i=1}^{n}$ are \textit{i.i.d.} samples drawn from an unknown distribution with conditional density $p(\bm{x}|y\in\mathcal{Y}_{s})$. Then we can get an equivalent risk expression of the classification risk (\ref{OR}), which is similar to the formulation (\ref{SC}):
\begin{theorem}\label{TF}
The multi-class classification risk can be recovered from data drawn from a subset of the collection of all the classes, i.e., $\mathcal{Y}_{s}\subset\mathcal{Y}$, if their confidence scores are given:
\begin{align}
\textstyle
R(\bm{g})=\pi_{\mathcal{Y}_{s}}\mathbb{E}_{p(\bm{x}|y\in\mathcal{Y}_{s})}\left[\sum_{y=1}^{K}\frac{r^{y}(\bm{x})}{r^{\mathcal{Y}_{s}}(\bm{x})}\ell(\bm{g}(\bm{x}),y)\right],
\end{align}
where $\pi_{\mathcal{Y}_{s}}=\sum_{y_{\mathrm{s}}\in\mathcal{Y}_{s}}\pi_{y_{\mathrm{s}}}$ and $r^{\mathcal{Y}_{s}}(\bm{x})=\sum_{y_{\mathrm{s}}\in\mathcal{Y}_{s}}r^{y_{\mathrm{s}}}(\bm{x})$.
\end{theorem}
The proof is shown in Appendix F. With this conclusion, we can approximate the classification risk (\ref{OR}) with data from a subset of all the classes as:
\begin{align}\textstyle\label{A}
\hat{R}_{\mathrm{sub}}(\bm{g})=\frac{\pi_{\mathcal{Y}_{s}}}{n}\sum_{i=1}^{n}\sum_{y=1}^{K}\frac{r^{y}_{i}}{r^{\mathcal{Y}_{s}}_{i}}\ell(\bm{g}(\bm{x}_{i}),y),
\end{align}
where $r^{\mathcal{Y}_{s}}_{i}=\sum_{y_{\mathrm{s}}\in\mathcal{Y}_{s}}r^{y_{\mathrm{s}}}_{i}$. Then it is routine to conduct ERM based on the unbiased risk estimator (\ref{A}). When $|\mathcal{Y}_{s}|=1$, the Sub-Conf classification is equivalent to the SC-Conf classification, which indicates that Sub-Conf classification is the generalization of SC-Conf classification. Similarly, the class-prior probability $\pi_{\mathcal{Y}_{s}}$ can be eliminated in the process of ERM.
\begin{remark}
{\rm A potential way of constructing equivalent risk expression with data from classes $\mathcal{Y}_{s}\subset\mathcal{Y}$ is the convex combination of SC-Conf classification:
$$R(\bm{g})=\sum\nolimits_{y_{\mathrm{s}}\in\mathcal{Y}_{s}}\alpha_{y_{\mathrm{s}}}\pi_{y_{\mathrm{s}}}\mathbb{E}_{p(\bm{x}|y_{\mathrm{s}})}\left[\sum\nolimits_{y=1}^{K}\frac{r^{y}(\bm{x})}{r^{y_{\mathrm{s}}}(\bm{x})}\ell(\bm{g}(\bm{x}),y)\right],~~{\rm s.t.}~\sum\nolimits_{y_{\mathrm{s}}\in\mathcal{Y}_{s}}\alpha_{y_{\mathrm{s}}}=1,~~\alpha_{y_{\mathrm{s}}}\geq 0,$$
which is a straightforward extension of SC-Conf classification. However, this method requires that we have the exact label $y_{\mathrm{s}}\in\mathcal{Y}_{s}$ for each instance. In contrast, it can be seen from the formulation of the empirical risk (\ref{A}) that the exact class label for each instance is not necessary in our Sub-Conf classification method, which shows that our Sub-Conf classification method is non-trivial.} 
\end{remark}

It can be seen that Sub-Conf classification and SC-Conf classification share the nearly identical risk formulation and data generation process and we can get a similar conclusion for Sub-Conf classification on its convergence analysis by substituting $y_{\mathrm{s}}$ with $\mathcal{Y}_{s}$. Given space limitations, we omit the convergence analysis of Sub-Conf classification and show the specific formulation of Noise-Robust Sub-Conf classification in the Appendix G.  
\section{Experiments}
\label{se}
In this section, we experimentally show the usefulness of our proposed methods for training deep models on two benchmark datasets. The implementation is conducted on
NVIDIA GeForce RTX 3090 GPUs based on Pytorch \cite{pytorch} and Sklean \cite{sklearn}. \vspace{-7pt}
\paragraph{Baselines:}In the experiments, SC-Conf, NoRSC-Conf, and Sub-Conf are short for ERM with risk estimators (\ref{USC}), (\ref{ERM-SC}), and (\ref{A}). We compare our methods to a simple weighted classification estimator: 
\begin{align}
\textstyle
\label{weight}
\hat{R}_{w}(\bm{g})=\frac{1}{n}\sum_{i=1}^{n}\sum_{y=1}^{K}r^{y}_{i}\ell(\bm{g}(\bm{x}_{i}),y).
\end{align}
This binary version of this estimator was used as a baseline in Pconf classification \cite{Pconf}. We regard ERM using this risk estimator with Weighted in the following parts. Though this estimator is a natural idea of utilizing the confidences when the data is collected from all the classes, it is biased generally in our setting where only data of a single class or a subset of all the classes are available. We also offer the result of learning with fully-supervised data. The details of the used datasets and our experimental setups on them are shown in this section.
\vspace{-7pt}
\paragraph{Datasets:} We evaluate the performance of the proposed methods and baselines on \textit{Fasion-MNIST} dataset \cite{F} and \textit{CIFAR-10} dataset \cite{C}. In the experiments, we use data of a single class or a subset of all the classes to train a classifier that can conduct classification successfully on all the classes. The detailed information of the used datasets are listed in the supplementary materials.
\vspace{-7pt}
\paragraph{Setup:} We briefly introduce the setting of our experiments and the detailed statistics are shown in the supplementary materials. We train the proposed methods and baseline methods with 3-layer multi-layer perceptron and DenseNet-161 \cite{densenet} with softmax cross-entropy loss on Fashion-MNIST and CIFAR-10, respectively. Adam \cite{adam} is used as the optimization algorithm. The validation accuracy of the used methods are calculated according to their empirical risk estimators on a validation set consisted of SC/Sub-Conf data. Then we do not have to collect additional labeled data for validation. 


Though we ask annotators for the values of confidences in the real-world application, we simulate the confidences here by a probabilistic model. For generating the confidences, we use the separated labeled dataset for training a probabilistic model with the same loss function as in the last paragraph and this dataset is separated from any other process of experiments. Since the minimizer of cross-entropy loss is a good estimator of the class-posterior probability \cite{BComp, Deterministic_Partial, Consistent_Partial}, we use the output of the training and validation sets after a softmax layer as their confidences. 


As to the noisy confidences, we consider the scenario in Remark \ref{r1}. Since most of the elements are zero in this scenario, the SC/Sub-Conf methods are not applicable here since the denominators in (\ref{USC}) and (\ref{A}) may be zeros. We conduct experiments on the proposed NoRSC-Conf method and the baseline weighted method to show the robustness of NoRSC-Conf against extreme noise.

\begin{table}[t]
\caption{Mean and standard deviation of the classification accuracy over 10 trials for the Fashion-MNIST dataset. The proposed methods were compared with the baseline Weighted method and fully-supervised method, with different classes used for training. The best and equivalent methods are shown in bold based on the 5\% t-test, excluding fully-supervised method.}
    \label{T1}
    \centering
\resizebox{0.95\textwidth}{!}{
    \begin{tabular}{c|c|cc|c|c}
    \toprule
    \multicolumn{2}{c|}{Used Classes}&SC/Sub-Conf& NoRSC-Conf&Weighted&Supervised\\
    \midrule
    \multirow{2}*{Coat}&Accurate&\textbf{54.66$\pm$2.17}&49.72$\pm$2.35&49.63$\pm$1.82&\multirow{14}*{\centering 80.13$\pm$2.75}\\
    
    \cmidrule{2-5} &Noisy&--$\pm$--&\textbf{49.64$\pm$3.06}&\textbf{49.29$\pm$3.87}\\
    \cmidrule{1-5}
    \multirow{2}*{Sandal}&Accurate&\textbf{56.02$\pm$4.55}&45.65$\pm$6.42&44.59$\pm$7.98\\
    
    \cmidrule{2-5} &Noisy&--$\pm$--&\textbf{44.50$\pm$5.19}&40.28$\pm$4.85\\
    \cmidrule{1-5}
    \multirow{2}*{Shirt}&Accurate& \textbf{71.44$\pm$4.77}&60.26$\pm$7.45&57.90$\pm$7.46\\
    
    \cmidrule{2-5} &Noisy&--$\pm$--&\textbf{59.16$\pm$4.51}&54.23$\pm$4.33\\
    \cmidrule{1-5}
    \multirow{2}*{Bag}&Accurate&\textbf{71.29$\pm$2.99}&66.04$\pm$2.03&68.31$\pm$1.56\\
    
    \cmidrule{2-5} &Noisy&--$\pm$--&\textbf{63.19$\pm$4.24}&60.67$\pm$2.70\\
    \cmidrule{1-5}
    \multirow{2}*{Trouser \& Pullover}&Accurate&\textbf{62.05$\pm$2.92}&54.25$\pm$4.17&53.86$\pm$4.68\\
    
    \cmidrule{2-5} &Noisy&--$\pm$--&\textbf{53.74$\pm$2.73}&50.56$\pm$3.87\\
    \cmidrule{1-5}
    \multirow{2}*{Dress \& Sneaker \&Ankle Boot}&Accurate&\textbf{76.96$\pm$2.03}&74.64$\pm$2.38&73.06$\pm$2.43\\
    
    \cmidrule{2-5} &Noisy&--$\pm$--&\textbf{72.14$\pm$2.57}&\textbf{71.42$\pm$3.23}\\
    \bottomrule
    \end{tabular}
}
\vspace{-14pt}
\end{table}
\begin{table}[!htbp]
\caption{Mean and standard deviation of the classification accuracy over 5 trials for the CIFAR-10 dataset. The proposed methods were compared with the baseline Weighted method and fully-supervised method, with different classes used for training. The best and equivalent methods are shown in bold based on the 5\% t-test, excluding fully-supervised method.}
    \label{TB2}
    \centering
\resizebox{0.95\textwidth}{!}{
    \begin{tabular}{c|c|cc|c|c}
    \toprule
    \multicolumn{2}{c|}{Used Classes}&SC/Sub-Conf& NoRSC-Conf&Weighted&Supervised\\
    \midrule
    \multirow{2}*{Dog}&Accurate&\textbf{53.10$\pm$0.93}&51.52$\pm$0.76&51.60$\pm$0.98&\multirow{12}*{\centering 65.62$\pm$0.84}\\
    
    \cmidrule{2-5} &Noisy&--$\pm$--&\textbf{41.77$\pm$0.70}&39.94$\pm$0.74\\
    \cmidrule{1-5}
    \multirow{2}*{Airplane \& Dog}&Accurate& \textbf{56.73$\pm$0.42}&55.69$\pm$0.68&54.47$\pm$1.21\\
    
    \cmidrule{2-5} &Noisy&--$\pm$--&\textbf{46.47$\pm$0.59}&\textbf{45.42$\pm$0.83}\\
    \cmidrule{1-5}
    \multirow{2}*{Airplane \& Deer \& Ship}&Accurate&\textbf{57.12$\pm$1.63}&53.69$\pm$1.70&54.68$\pm$0.64\\
    
    \cmidrule{2-5} &Noisy&--$\pm$--&\textbf{51.11$\pm$0.65}&48.96$\pm$1.00\\
    \cmidrule{1-5}
    \multirow{2}*{\shortstack{Airplane \& Cat \& Frog\\ \& Truck}}&Accurate&\textbf{60.11$\pm$0.27}&58.95$\pm$0.61&57.47$\pm$1.26\\
    
    \cmidrule{2-5} &Noisy&--$\pm$--&\textbf{54.06$\pm$0.47}&52.56$\pm$0.85\\
    \cmidrule{1-5}
    \multirow{2}*{\shortstack{Airplane \& Cat \& Frog\\ Ship \& Truck}}&Accurate&\textbf{63.25$\pm$1.53}&\textbf{61.19$\pm$0.95}&59.89$\pm$1.43\\
    
    \cmidrule{2-5} &Noisy&--$\pm$--&\textbf{58.95$\pm$1.05}&54.42$\pm$2.38\\
    \bottomrule
    \end{tabular}
}
\vspace{-14pt}
\end{table}
\vspace{-7pt}
\paragraph{Results:} The experimental results are shown in Table \ref{T1} and \ref{TB2}. As we can see, when accurate confidences are used, the SC/Sub-Conf outperform NoRSC-Conf and Weighted in almost all the cases. When the size of the subset of all the classes increasing, the learning result of SC/Sub-Conf is even comparable to that of fully-supervised learning. Notice that since we only have data from a single class or a subset of all the classes, the training samples used for SC/Sub-Conf learning are far less than the whole training set. For a fair comparison, the number of training samples used for fully-supervised learning is reduced to that of the subset with the most samples.

When only the noisy confidences are given, NoRSC-Conf outperforms the weighted classification baseline in eight cases and is comparable to it in three cases. This result shows that NoRSC-Conf can alleviate the effect of noisy supervision and identify the Bayes-optimal classifier, which aligns with our theoretical analysis in Section \ref{S4}.
\vspace{-7pt}
\section{Conclusion}
\vspace{-7pt}
In this paper, we propose a novel weakly supervised learning setting and effective algorithms for provably consistent multi-class classification from data of a single class or a subset of all the classes equipped with confidences. We make three key contributions in this paper. Firstly, we propose an unbiased risk estimator for multi-class classification from data of a single class with confidences and provided the estimation error analysis. Secondly, we theoretically show that the proposed method can be robust to extreme noise and converge to the Bayes-optimal classifier in probability, and both infinite and finite-sample analyses on the misclassification rate are given. Finally, we extend our method to the case where data from a subset of all the classes are available. The experimental results demonstrate the usefulness of our algorithms.
\paragraph{Acknowledgments} This work was supported by the National Natural Science Foundation of China (Nos. 12071475, 11671010), Beijing Natural Science Foundation (No.4172035). GN was supported by JST AIP Acceleration Research Grant Number JPMJCR20U3, Japan. MS was supported by JST CREST Grant Number JPMJCR18A2, Japan.


\newpage
\appendix
\section{Proof of Theorem \ref{URE-SC}}
\begin{proof}
\begin{align*}
\pi_{y_{s}}\mathbb{E}_{p(\bm{x}|y_{s})}\left[\sum_{y=1}^{K}\frac{r^{y}(\bm{x})}{r^{y_{s}}(\bm{x})}\ell(\bm{g}(\bm{x}),y)\right]&=\pi_{y_{s}}\int\sum_{y=1}^{K}\frac{p(y|\bm{x})}{p(y_{s}|\bm{x})}\ell(\bm{g}(\bm{x}),y)p(\bm{x}|y_{s})d\bm{x}\\
&=\int\sum_{y=1}^{K}\frac{p(\bm{x},y_{s})}{p(y_{s}|\bm{x})}\ell(\bm{g}(\bm{x}),y)p(y|\bm{x})d\bm{x}\\
&=\int\sum_{y=1}^{K}p(\bm{x})\ell(\bm{g}(\bm{x}),y)p(y|\bm{x})d\bm{x}\\
&=\int\sum_{y=1}^{K}\ell(\bm{g}(\bm{x}),y)p(\bm{x},y)d\bm{x}=R(\bm{g}).
\end{align*}
\end{proof}
\section{Proof of Theorem \ref{Bound}}
To prove the Theorem \ref{Bound}, we first give the definition of Rademacher complexity:
\begin{definition}(Rademacher complexity) Let $Z_{1},\cdots,Z_{n}$ be n \textit{i.i.d.} random variables drawn from a probability distribution $\mu$ and $\mathcal{F}=\{f : Z \rightarrow \mathbb{R}\}$ be a class of measurable functions. Then the expected Rademacher complexity of function class $\mathcal{F}$ is given by:
\begin{align}
\mathfrak{R}_{n}(\mathcal{F})=\mathbb{E}_{Z_{1},\cdots,Z_{n}\sim\mu}\mathbb{E}_{\bm{\sigma}}\left[{\rm sup}_{f\in\mathcal{F}}\frac{1}{n}\sum_{i=1}^{n}\sigma_{i}f(Z_{i})\right],
\end{align}
where $\sigma_{1},\cdots,\sigma_{n}$ are the Rademacher variables that take the value from $\{-1, +1\}$ evenly. 
\end{definition}
Based on the setting and assumptions in Section \ref{sbound}, we have the following lemma for bounding the uniform convergence:
\begin{lemma} For any $0<\delta<1$, the following inequality holds with  probability at least $1-\delta$:
\begin{align}
{\rm sup_{\bm{g}\in{\bm{\mathcal{G}}}}}\left|\hat{R}_{SC}(\bm{g})-R(\bm{g})\right|\leq \frac{2\sqrt{2}\pi_{y_{s}}L_{\ell}}{C_{r}}\sum_{y=1}^{K}\mathfrak{R}_{n}(\mathcal{G}_{y})+\frac{\pi_{y_{s}}C_{\ell}}{C_{r}}\sqrt{\frac{\log\frac{2}{\delta}}{2n}},
\end{align}
where $\mathfrak{R}_{n}(\mathcal{G}_{y})$ is the Rademacher complexity of $\mathcal{G}_{y}$ on \textit{i.i.d.} samples with size $n$ drawn from a single class $p(\bm{x}|y_{s})$.
\end{lemma}
\begin{proof}
We begin with proving that the one direction $\sup_{\bm{g}\in\bm{\mathcal{G}}}\hat{R}_{SC}(\bm{g})-R(\bm{g})$ is bounded with probability at least $1-\frac{\delta}{2}$. Suppose an instance $\bm{x}_{i}$ is changed by $\bm{x}_{i}'$, we can see that then change of $\sup_{\bm{g}\in\bm{\mathcal{G}}}\hat{R}_{SC}(\bm{g})-R(\bm{g})$ is no greater than $\pi_{y_{s}}C_{\ell}/nC_{r}$ . By applying the McDiarmid's inequality \cite{MC}, with probability at least $1-\frac{\delta}{2}$, the following inequality holds:
\begin{align*}
\sup_{\bm{g}\in\bm{\mathcal{G}}}\hat{R}_{SC}(\bm{g})-R(\bm{g})\leq \mathbb{E}_{\bm{x}_{1},\cdots,\bm{x}_{n}}\left[\sup_{\bm{g}\in\bm{\mathcal{G}}}\hat{R}_{SC}(\bm{g})-R(\bm{g})\right]+\frac{\pi_{y_{s}}C_{\ell}}{C_{r}}\sqrt{\frac{\log\frac{2}{\delta}}{2n}}.
\end{align*}
Denote by $\mathcal{L}(\bm{g}(\bm{x}))=\sum_{y=1}^{K}\frac{r^{y}}{r^{y_{s}}}*\ell(\bm{g}(\bm{x}),y)$. It is easy to show that $\mathcal{L}(\bm{g}(\bm{x}))$ is  $\frac{L_{\ell}}{C_{r}}$-Lipschitz \textit{w.r.t.} $\bm{g}(\bm{x})$ due to the fact that $\sum_{y=1}^{K}r^{y}=1$ and $r^{y_{s}}\geq C_{r}$. Since $\hat{R}_{SC}(\bm{g})$ is unbiased, it is routine to show that \cite{FML}:
\begin{align*}
\mathbb{E}_{\bm{x}_{1},\cdots,\bm{x}_{n}}\left[\sup_{\bm{g}\in\bm{\mathcal{G}}}\hat{R}_{SC}(\bm{g})-R(\bm{g})\right]&\leq 2\pi_{y_{s}}\mathfrak{R}_{n}\left(\mathcal{L}\circ\bm{\mathcal{G}}\right)\\
&\leq \frac{2\sqrt{2}\pi_{y_{s}}L_{\ell}}{C_{r}}\sum_{y=1}^{K}\mathfrak{R}_{n}(\mathcal{G}_{y}),
\end{align*}
The last inequality holds according to the Talagrand's contraction inequality \cite{tala}.

The proof of the other direction $\sup_{\bm{g}\in\bm{\mathcal{G}}}R(\bm{g})-\hat{R}_{SC}(\bm{g})$ is similar. Thus we conclude the proof.
\end{proof}
Then we can begin to prove the Theorem \ref{Bound}:
\begin{proof}
\begin{align*}
R(\hat{\bm{g}}_{SC})-R(\bm{g}^{*})&=\left(R(\hat{\bm{g}}_{SC})-\hat{R}_{SC}(\hat{\bm{g}}_{SC})\right)+\left(\hat{R}_{SC}(\hat{\bm{g}}_{SC})-\hat{R}_{SC}(\bm{g}^{*})\right)+\left(\hat{R}_{SC}(\bm{g}^{*})-R(\bm{g}^{*})\right)\\
&\leq \left(R(\hat{\bm{g}}_{SC})-\hat{R}_{SC}(\hat{\bm{g}}_{SC})\right)+\left(\hat{R}_{SC}(\bm{g}^{*})-R(\bm{g}^{*})\right)\\
&\leq 2\sup_{\bm{g}\in\bm{\mathcal{G}}}\left|R(\hat{\bm{g}})-\hat{R}_{SC}(\hat{\bm{g}})\right|,
\end{align*}
where the second inequality holds according to the definition of ERM. Combining this conclusion with the previous lemma, we can conclude the proof.
\end{proof}
\section{Proof of Theorem \ref{T2}}
\begin{proof}
First of all, we give the definition of classification-calibrated loss:
\begin{proposition}
\label{P1}{\rm \cite{MCC}}
If the classification-calibrated losses are used, the optimal classifier $\bm{g}^{*}$ that minimizes the classification risk among all the measurable functions is also Bayes-optimal classifier: the classifier that minimizes the classification risk \textit{w.r.t.} 0-1 loss, i.e., misclassification rate. In other words, the prediction of the classifier on $\bm{x}$ satisfies this condition: $\mathop{\rm argmax}_{y\in\mathcal{Y}}g_{y}(\bm{x})=\mathop{\rm argmax}_{y\in\mathcal{Y}}p(y|\bm{x})$..
\end{proposition}
Denote by $\tilde{p}(\bm{x},y)=\tilde{p}(y|\bm{x})p(\bm{x})$. Notice that the following holds:
\begin{align*}
\int\sum_{y=1}^{K}\tilde{p}(\bm{x},y)d\bm{x}&=\int\sum_{y=1}^{K}\tilde{p}(y|\bm{x})p(\bm{x})d\bm{x}\\
&=\int\left(\sum_{y=1}^{K}\tilde{p}(y|\bm{x})\right)p(\bm{x})d\bm{x}\\
&=\int p(\bm{x})d\bm{x}=1
\end{align*}
Then we can see that $\tilde{p}(\bm{x},y)$ is the density of a valid joint distribution on $\mathcal{X}\times\mathcal{Y}$. Then we have:
\begin{align*}
\tilde{R}(\bm{g})&=\mathbb{E}_{p(\bm{x}|y_{s})}\left[\phi(\bm{x})\sum_{y=1}^{K}\tilde{r}^{y}(\bm{x})\ell(\bm{g}(\bm{x}),y)\right]\\
&=\int p(\bm{x}|y_{s})\phi(\bm{x})\sum_{y=1}^{K}\tilde{r}^{y}(\bm{x})\ell(\bm{g}(\bm{x}),y)d\bm{x}\\
&=\int p(\bm{x})\sum_{y=1}^{K}\tilde{r}^{y}(\bm{x})\ell(\bm{g}(\bm{x}),y)d\bm{x}\\
&=\int\sum_{y=1}^{K}\tilde{p}(\bm{x},y)\ell(\bm{g}(\bm{x}),y)d\bm{x}\\
&=\mathbb{E}_{\tilde{p}(\bm{x},y)}\left[\ell(\bm{g}(\bm{x},y))\right]
\end{align*}
Let $\tilde{\bm{g}}^{*}={\rm argmin}_{\bm{g}\in\bm{\mathcal{G}}}\tilde{R}(\bm{g})$. According to the definition of the classification-calibrated loss, we can see that $\tilde{\bm{g}}^{bayes}=\mathop{\rm argmin}_{\bm{g}~{\rm measurable}}\tilde{R}(\bm{g})$ is Bayes-optimal. Since $\tilde{\bm{g}}^{bayes}$ is in $\mathcal{\bm{G}}$, we have that $\tilde{R}(\tilde{\bm{g}}^{bayes})\geq\tilde{R}(\tilde{\bm{g}}^{*})$. However, since $\tilde{R}(\tilde{\bm{g}}^{bayes})=\min_{\bm{g}~{\rm measurable}}\tilde{R}(\bm{g})$, we have that $\tilde{R}(\tilde{\bm{g}}^{bayes})\leq\tilde{R}(\tilde{\bm{g}}^{*})$. Combining the two conclusions, we know that $\tilde{R}(\tilde{\bm{g}}^{bayes})=\tilde{R}(\tilde{\bm{g}}^{*})$. Then we can see that $\tilde{\bm{g}}^{*}={\rm argmin}_{\bm{g}~{\rm measurable}}\tilde{R}(\bm{g})$, which indicates that $\tilde{\bm{g}}^{*}$ is also Bayes-optimal. 

According to the definition of the Bayes-optimal classifier, we have that ${\rm argmax}_{y\in\mathcal{Y}}\tilde{g}_{y}^{*}(\bm{x})={\rm argmax}_{y\in\mathcal{Y}}\tilde{p}(y|\bm{x})$. Since $\mathop{\rm argmax}_{y\in\mathcal{Y}}\tilde{p}(y|\bm{x})=\mathop{\rm argmax}_{y\in\mathcal{Y}}p(y|\bm{x})$, we have that ${\rm argmax}_{y\in\mathcal{Y}}\tilde{g}_{y}^{*}(\bm{x})={\rm argmax}_{y\in\mathcal{Y}}p(y|\bm{x})$. Then we can see that $\tilde{\bm{g}}^{*}$ is also the Bayes-optimal classifier of $R(\bm{g})$, which shows the classifier-consistency of $\tilde{R}(\bm{g})$.
\end{proof}
\section{Proof of Lemma \ref{lbound} and Theorem \ref{final}}
We begin with the proof of Lemma \ref{lbound}:
\begin{proof}
We prove this lemma by bounding the uniform convergence as in the proof of Theorem \ref{Bound}. First we prove the following technical lemma:
\begin{lemma}
For any $0<\delta<1$, the following inequality holds with  probability at least $1-\delta$:
\begin{align}
{\rm sup_{\bm{g}\in{\bm{\mathcal{G}}}}}\left|\hat{\tilde{{R}}}(\bm{g})-\tilde{R}(\bm{g})\right|\leq\frac{2\sqrt{2}L_{\ell}}{C_{r}}\sum_{y=1}^{K}\mathfrak{R}_{n}(\mathcal{G}_{y})+\frac{C_{\ell}}{C_{r}}\sqrt{\frac{\log\frac{2}{\delta}}{2n}},
\end{align}
where $\mathfrak{R}_{n}(\mathcal{G}_{y})$ is the Rademacher complexity of $\mathcal{G}_{y}$ on \textit{i.i.d.} samples with size $n$ drawn from a single class $p(\bm{x}|y_{s})$.
\end{lemma}
\begin{proof}
We only prove that the direction $\sup_{\bm{g}\in\bm{\mathcal{G}}}\hat{\tilde{{R}}}(\bm{g})-\tilde{R}(\bm{g})$ is bounded with probability at least $1-\frac{\delta}{2}$ since the proof of another direction is completely symmetric. Suppose an instance $\bm{x}_{i}$ is changed by $\bm{x}_{i}'$, we can see that then change of $\sup_{\bm{g}\in\bm{\mathcal{G}}}\hat{\tilde{{R}}}(\bm{g})-\tilde{R}(\bm{g})$ is no greater than $C_{\ell}/nC_{r}$ since $\phi_{i}\leq 1/C_{r}$. By applying the McDiarmid's inequality \cite{MC}, with probability at least $1-\frac{\delta}{2}$, the following inequality holds:
\begin{align*}
\sup_{\bm{g}\in\bm{\mathcal{G}}}\hat{\tilde{{R}}}(\bm{g})-\tilde{R}(\bm{g})\leq\mathbb{E}_{\bm{x}_{1},\cdots,\bm{x}_{n}}\left[\sup_{\bm{g}\in\bm{\mathcal{G}}}\hat{\tilde{{R}}}(\bm{g})-\tilde{R}(\bm{g})\right]+\frac{C_{\ell}}{C_{r}}\sqrt{\frac{\log\frac{2}{\delta}}{2n}}.
\end{align*}
Denote by $\tilde{\mathcal{L}}(\bm{g}(\bm{x}))=\phi(\bm{x})\sum_{y=1}^{K}\tilde{r}^{y}*\ell(\bm{g}(\bm{x}),y)$. It is easy to show that $\mathcal{L}(\bm{g}(\bm{x}))$ is  $\frac{L_{\ell}}{C_{r}}$-Lipschitz \textit{w.r.t.} $\bm{g}(\bm{x})$ due to the fact that $\sum_{y=1}^{K}r^{y}=1$ and $\phi(\bm{x})\leq 1/C_{r}$. Then it is routine to show that \cite{FML}:
\begin{align*}
\mathbb{E}_{\bm{x}_{1},\cdots,\bm{x}_{n}}\left[\sup_{\bm{g}\in\bm{\mathcal{G}}}\hat{\tilde{{R}}}(\bm{g})-\tilde{R}(\bm{g})\right]&\leq 2\mathfrak{R}_{n}\left(\mathcal{L}\circ\bm{\mathcal{G}}\right)\\
&\leq \frac{2\sqrt{2}L_{\ell}}{C_{r}}\sum_{y=1}^{K}\mathfrak{R}_{n}(\mathcal{G}_{y}).
\end{align*}
The last inequality holds according to the Talagrand's contraction inequality \cite{tala}.

The proof of the other direction is similar. Thus we conclude the proof.
\end{proof}
Then we can begin to prove Lemma \ref{lbound} as in the proof of Theorem \ref{Bound}:
\begin{align*}
\tilde{R}(\hat{\tilde{\bm{g}}})-\tilde{R}(\tilde{\bm{g}}^{*})&=\left(\tilde{R}(\hat{\tilde{\bm{g}}})-\hat{\tilde{R}}(\hat{\tilde{\bm{g}}})\right)+\left(\hat{\tilde{R}}(\hat{\tilde{\bm{g}}})-\hat{\tilde{R}}(\tilde{\bm{g}}^{*})\right)+\left(\hat{\tilde{R}}(\tilde{\bm{g}}^{*})-\tilde{R}(\tilde{\bm{g}}^{*})\right)\\
&\leq \left(\tilde{R}(\hat{\tilde{\bm{g}}})-\hat{\tilde{R}}(\hat{\tilde{\bm{g}}})\right)+\left(\hat{\tilde{R}}(\tilde{\bm{g}}^{*})-\tilde{R}(\tilde{\bm{g}}^{*})\right)\\
&\leq 2\sup_{\bm{g}\in\bm{\mathcal{G}}}\left|\tilde{R}(\hat{\bm{g}})-\hat{\tilde{R}}(\hat{\bm{g}})\right|,
\end{align*}
where the second inequality holds according to the definition of ERM. Combining this conclusion with the previous lemma, we can conclude the proof.
\end{proof}
Then we prove the Theorem \ref{final}:
\begin{proof}
First we reformulate the expression of $\tilde{R}_{01}(\hat{\tilde{\bm{g}}})-\tilde{R}_{01}(\tilde{\bm{g}}^{*})$. 
According to the proof of Theorem \ref{T2}, $\tilde{\bm{g}}^{*}$ is the Bayes-optimal classifier of $\tilde{R}(\bm{g})$. Denote by $f(\bm{x})={\rm argmax}_{y\in\mathcal{Y}}\hat{\tilde{g}}_{y}(\bm{x})$ the decision function \textit{w.r.t.} $\hat{\tilde{\bm{g}}}$. Using the definition of Bayes-optimal classifier, we have the following equations:
\begin{align*}
\tilde{R}_{01}(\hat{\tilde{\bm{g}}})-\tilde{R}_{01}(\tilde{\bm{g}}^{*})&=\mathbb{E}_{\tilde{p}(\bm{x},y)}\left[1(f(\bm{x}))\not=y)\right]-\mathbb{E}_{\tilde{p}(\bm{x},y)}\left[1({\rm argmax}_{y\in\mathcal{Y}}\tilde{g}_{y}^{*}\not=y)\right]\\
&=\mathbb{E}_{p(\bm{x})}\left[\mathbb{E}_{\tilde{p}(y|\bm{x})}\left[1(f(\bm{x}))\not=y)-1({\rm argmax}_{y\in\mathcal{Y}}\tilde{g}_{y}^{*}\not=y)\right]\right]\\
&=\mathbb{E}_{p(\bm{x})}\left[1-\tilde{p}(f(\bm{x})|\bm{x})-\left(1-\max_{y\in\mathcal{Y}}\tilde{p}(y|\bm{x})\right)\right]\\
&=\mathbb{E}_{p(\bm{x})}\left[\max_{y\in\mathcal{Y}}\tilde{p}(y|\bm{x})-\tilde{p}(f(\bm{x})|\bm{x})\right]\\
&=\mathbb{E}_{p(\bm{x})}\left[\max_{y\in\mathcal{Y}}\tilde{p}(y|\bm{x})-\tilde{p}(f(\bm{x})|\bm{x})\right]\\
&=\mathbb{E}_{p(\bm{x})}\left[1(f(\bm{x})\not={\rm argmax}_{y\in\mathcal{Y}}\tilde{p}(y|\bm{x}))\left(\max_{y\in\mathcal{Y}}\tilde{p}(y|\bm{x})-\tilde{p}(f(\bm{x})|\bm{x})\right)\right].
\end{align*}
Since $\inf_{\bm{x}\in\mathcal{X}}\Delta(\bm{x})\leq \left(\max_{y\in\mathcal{Y}}\tilde{p}(y|\bm{x})-\tilde{p}(f(\bm{x})|\bm{x})\right)$ almost surely for all the $\bm{x}\in\mathcal{X}$, we have the following inequalities:
$$
\left\{ 
\begin{matrix}
1(f(\bm{x})\not={\rm argmax}_{y\in\mathcal{Y}}\tilde{p}(y|\bm{x}))\left(\max_{y\in\mathcal{Y}}\tilde{p}(y|\bm{x})-\tilde{p}(f(\bm{x})|\bm{x})\right)=1(f(\bm{x})\not={\rm argmax}_{y\in\mathcal{Y}}\tilde{p}(y|\bm{x}))*\inf_{\bm{x}\in\mathcal{X}}\Delta(\bm{x}), \\~~~~\quad\quad\quad\quad\quad\quad\quad\quad\quad\quad\quad\quad\quad\quad\quad\quad\quad\quad\quad\quad\quad\quad\quad\quad\quad\quad\quad\quad\quad\quad\quad\quad(f(\bm{x})={\rm argmax}_{y\in\mathcal{Y}}\tilde{p}(y|\bm{x})).\\
1(f(\bm{x})\not={\rm argmax}_{y\in\mathcal{Y}}\tilde{p}(y|\bm{x}))\left(\max_{y\in\mathcal{Y}}\tilde{p}(y|\bm{x})-\tilde{p}(f(\bm{x})|\bm{x})\right)>1(f(\bm{x})\not={\rm argmax}_{y\in\mathcal{Y}}\tilde{p}(y|\bm{x}))*\inf_{\bm{x}\in\mathcal{X}}\Delta(\bm{x}), \\~~~~\quad\quad\quad\quad\quad\quad\quad\quad\quad\quad\quad\quad\quad\quad\quad\quad\quad\quad\quad\quad\quad\quad\quad\quad\quad\quad\quad\quad\quad\quad\quad\quad(f(\bm{x})\not={\rm argmax}_{y\in\mathcal{Y}}\tilde{p}(y|\bm{x})).
\end{matrix}
\right.
$$
Then we have the following inequality:
$$1(f(\bm{x})\not={\rm argmax}_{y\in\mathcal{Y}}\tilde{p}(y|\bm{x}))\left(\max_{y\in\mathcal{Y}}\tilde{p}(y|\bm{x})-\tilde{p}(f(\bm{x})|\bm{x})\right)\geq1(f(\bm{x})\not={\rm argmax}_{y\in\mathcal{Y}}\tilde{p}(y|\bm{x}))*\inf_{\bm{x}\in\mathcal{X}}\Delta(\bm{x})$$
Then we can further give the lower bound of $\tilde{R}_{01}(\hat{\tilde{\bm{g}}})-\tilde{R}_{01}(\tilde{\bm{g}}^{*})$:
\begin{align*}
\tilde{R}_{01}(\hat{\tilde{\bm{g}}})-\tilde{R}_{01}(\tilde{\bm{g}}^{*})&=\mathbb{E}_{p(\bm{x})}\left[1(f(\bm{x})\not={\rm argmax}_{y\in\mathcal{Y}}\tilde{p}(y|\bm{x}))\left(\max_{y\in\mathcal{Y}}\tilde{p}(y|\bm{x})-\tilde{p}(f(\bm{x})|\bm{x})\right)\right]\\
&\geq\mathbb{E}_{p(\bm{x})}\left[1(f(\bm{x})\not={\rm argmax}_{y\in\mathcal{Y}}\tilde{p}(y|\bm{x}))*\inf_{\bm{x}\in\mathcal{X}}\Delta(\bm{x})\right]
\end{align*}
Since the factor $\inf_{\bm{x}\in\mathcal{X}}\Delta(\bm{x})$ is irrelevant to $\bm{x}$, we can further give the following conclusion:
\begin{align}\label{tech}
\mathbb{E}_{p(\bm{x})}\left[1(f(\bm{x})\not={\rm argmax}_{y\in\mathcal{Y}}\tilde{p}(y|\bm{x}))\right]\leq \frac{1}{\inf_{\bm{x}\in\mathcal{X}}\Delta(\bm{x})}\left(\tilde{R}_{01}(\hat{\tilde{\bm{g}}})-\tilde{R}_{01}(\tilde{\bm{g}}^{*})\right).
\end{align}
Since $\tilde{\bm{g}}^{*}\in\bm{\mathcal{G}}$ is also the Bayes-optimal classifier of $R(\bm{g})$ as stated in the proof of Theorem \ref{T2}, we can learn that $R_{01}(\tilde{\bm{g}}^{*})\leq R_{01}(\bm{g}_{01}^{*})$. Notice that $\tilde{\bm{g}}^{*}$ is in $\bm{\mathcal{G}}$, then $R_{01}(\tilde{\bm{g}}^{*})\geq R_{01}(\bm{g}_{01}^{*})$. Combining the two inequalities, we can see that $R_{01}(\tilde{\bm{g}}^{*})=R_{01}a(\bm{g}_{01}^{*})$. Then we can get the expression of $R_{01}(\hat{\tilde{\bm{g}}})-R_{01}(\bm{g}_{01}^{*})$:
\begin{align*}
R_{01}(\hat{\tilde{\bm{g}}})-R_{01}(\bm{g}_{01}^{*})&= R_{01}(\hat{\tilde{\bm{g}}})-R_{01}(\tilde{\bm{g}}^{*})\\
&=\mathbb{E}_{p(\bm{x},y)}\left[1(f(\bm{x}))\not=y)\right]-\mathbb{E}_{p(\bm{x},y)}\left[1({\rm argmax}_{y\in\mathcal{Y}}\tilde{g}_{y}^{*}\not=y)\right]\\
&=\mathbb{E}_{p(\bm{x})}\left[\mathbb{E}_{p(y|\bm{x})}\left[1(f(\bm{x}))\not=y)-1({\rm argmax}_{y\in\mathcal{Y}}\tilde{g}_{y}^{*}\not=y)\right]\right]\\
&=\mathbb{E}_{p(\bm{x})}\left[1-p(f(\bm{x})|\bm{x})-\left(1-\max_{y\in\mathcal{Y}}p(y|\bm{x})\right)\right]\\
&=\mathbb{E}_{p(\bm{x})}\left[\max_{y\in\mathcal{Y}}p(y|\bm{x})-p(f(\bm{x})|\bm{x})\right]\\
&=\mathbb{E}_{p(\bm{x})}\left[\max_{y\in\mathcal{Y}}p(y|\bm{x})-p(f(\bm{x})|\bm{x})\right]\\
&=\mathbb{E}_{p(\bm{x})}\left[1(f(\bm{x})\not={\rm argmax}_{y\in\mathcal{Y}}p(y|\bm{x}))\left(\max_{y\in\mathcal{Y}}p(y|\bm{x})-p(f(\bm{x})|\bm{x})\right)\right].
\end{align*}
Since ${\rm argmax}_{y\in\mathcal{Y}}p(y|\bm{x})={\rm argmax}_{y\in\mathcal{Y}}\tilde{p}(y|\bm{x})$ and $\left(\max_{y\in\mathcal{Y}}p(y|\bm{x})-p(f(\bm{x})|\bm{x})\right)\leq 1$, we can get the following inequality:
$$R_{01}(\hat{\tilde{\bm{g}}})-R_{01}(\bm{g}_{01}^{*})\leq \mathbb{E}_{p(\bm{x})}\left[1(f(\bm{x})\not={\rm argmax}_{y\in\mathcal{Y}}\tilde{p}(y|\bm{x}))\right].$$
Combining the inequality above, (\ref{tech}), and Lemma \ref{tool}, we can conclude the proof of Theorem \ref{final}. 
\end{proof}
\section{Proof of Theorem \ref{final_bound}}
\begin{proof}
$\hat{B}^{SC}_{\eta}(\hat{\phi})=\frac{1}{n}\sum_{i=1}^{n}\left(\nabla\eta(\hat{\phi}(\bm{x}_{i}))\hat{\phi}(\bm{x}_{i})-\eta(\hat{\phi}(\bm{x}_{i}))\right)$ and $\hat{B}^{U}_{\eta}(\hat{\phi})=-\frac{1}{n_{u}}\sum_{i=1}^{n_{u}}\nabla\eta(\hat{\phi}(\bm{x}_{i}^{u}))$. $\hat{B}^{SC}_{\eta}(\hat{\phi})$ and $\hat{B}^{U}_{\eta}(\hat{\phi})$ are their expectations, respectively. Then the following work is to bound the uniform convergence. Assume that there is $C_{\hat{\phi}}>0$ that ${\rm sup}_{\hat{\phi}\in \Phi}\|\hat{\phi}\|_{\infty}\leq C_{\hat{\phi}}$. The function $\eta(\hat{\phi}(\bm{x}))$ is Lipschitz continuous for all $\|\hat{\phi}\|_{\infty}\leq C_{\hat{\phi}}$ with Lipschitz constant $L_{\eta}>0$ and upper-bounded by $C_{\eta}>0$. $\nabla\eta$ is also Lipschitz continuous with constant $L'_{\eta}>0$. Then we first give the technical lemma:

\begin{lemma} For any $0<\delta<1$, With probability at least 1-$\delta$: 
\begin{align}
&\sup_{\hat{\phi}\in\Phi}|\hat{B}^{SC}_{\eta}(\hat{\phi})-B^{SC}_{\eta}(\hat{\phi})|\leq 2L_{B}\mathfrak{R}_{n}(\Phi)+C_{B}\sqrt{\frac{\log\frac{2}{\delta}}{2n}}.\\
&\sup_{\hat{\phi}\in\Phi}|\hat{B}^{U}_{\eta}(\hat{\phi})-B^{U}_{\eta}(\hat{\phi})|\leq 2L'_{\eta}\mathfrak{R}_{n_{u}}^{u}(\Phi)+L_{\eta}\sqrt{\frac{\log\frac{2}{\delta}}{2n_{u}}}.
\end{align}
\end{lemma}
\begin{proof}
Notice that $\hat{B}^{SC}_{\eta}(\hat{\phi})$ is Lipschitz continuous with Lipschitz constant $L_{B}=L'_{\eta}C_{\hat{\eta}}+2L_{\eta}$ and its absolute value is bounded by $C_{B}=\max\{C_{\eta},~L_{\eta}C_{\hat{\phi}}\}$. By using the McDiarmid's inequality, Union bound, and the Talagrand's contraction inequality, we can learn that with probability at least 1-$\delta$:
\begin{align*}
\sup_{\hat{\phi}\in\Phi}|\hat{B}^{SC}_{\eta}(\hat{\phi})-B^{SC}_{\eta}(\hat{\phi})|\leq 2L_{B}\mathfrak{R}_{n}(\Phi)+C_{B}\sqrt{\frac{\log\frac{2}{\delta}}{2n}}.
\end{align*}

Notice that $\hat{B}^{U}_{\eta}(\hat{\phi})$ is $L'_{\eta}$-Lipschitz continuous and is bounded by $L_{\eta}$. In the same way, we can see that with probability at least 1-$\delta$:
\begin{align*}
\sup_{\hat{\phi}\in\Phi}|\hat{B}^{U}_{\eta}(\hat{\phi})-B^{U}_{\eta}(\hat{\phi})|\leq 2L'_{\eta}\mathfrak{R}_{n_{u}}^{u}(\Phi)+L_{\eta}\sqrt{\frac{\log\frac{2}{\delta}}{2n_{u}}}.
\end{align*}

\end{proof}
Then we can begin to prove the Theorem \ref{final_bound}:
\begin{align*}
    B_{\eta}(\phi\|\hat{\phi}^{*})&=B_{\eta}(\phi\|\hat{\phi}^{*})-B_{\eta}(\phi\|\phi)\\
    &=B_{\eta}(\phi\|\hat{\phi}^{*})-\hat{B}_{\eta}(\phi\|\hat{\phi}^{*})+\hat{B}_{\eta}(\phi\|\hat{\phi}^{*})-\hat{B}_{\eta}(\phi\|\phi)+\hat{B}_{\eta}(\phi\|\phi)-B_{\eta}(\phi\|\phi)\\
    &\leq 2\sup_{\hat{\phi}\in\Phi}|B_{\eta}(\phi\|\hat{\phi})-\hat{B}_{\eta}(\phi\|\hat{\phi})|\\
    &\leq 2\sup_{\hat{\phi}\in\Phi}\left(|\hat{B}^{SC}_{\eta}(\hat{\phi})-B^{SC}_{\eta}(\hat{\phi})|+|\hat{B}^{U}_{\eta}(\hat{\phi})-B^{U}_{\eta}(\hat{\phi})|\right)\\
    &\leq 2\sup_{\hat{\phi}\in\Phi}|\hat{B}^{SC}_{\eta}(\hat{\phi})-B^{SC}_{\eta}(\hat{\phi})|+2\sup_{\hat{\phi}\in\Phi}|\hat{B}^{U}_{\eta}(\hat{\phi})-B^{U}_{\eta}(\hat{\phi})|.
\end{align*}
According to the technical lemma above and the union bound, we can show that with probability at least $1-\delta$:
\begin{align*}
B_{\eta}(\phi\|\hat{\phi}^{*})\leq 4L_{B}\mathfrak{R}_{n}(\Phi)+2C_{B}\sqrt{\frac{\log\frac{4}{\delta}}{2n}}+4L'_{\eta}\mathfrak{R}_{n_{u}}^{u}(\Phi)+2L_{\eta}\sqrt{\frac{\log\frac{4}{\delta}}{2n_{u}}}.
\end{align*}
which concludes the proof.
\end{proof}
\section{Proof of Theorem \ref{TF} and the Formulation of Noise-Robust Sub-Conf Learning}
\begin{proof}
\begin{align*}
&\left(\sum_{y\in\mathcal{Y}_{s}}\pi_{y}\right)\mathbb{E}_{p(\bm{x}|y\in\mathcal{Y}_{s})}\left[\sum_{y=1}^{K}\left(\frac{p(y|\bm{x})}{\sum_{y_{s}\in\mathcal{Y}_{s}}p(y_{s}|\bm{x})}\ell(\bm{g}(\bm{x}),y)\right)\right]\\
&=\left(\sum_{y\in\mathcal{Y}_{s}}\pi_{y}\right)\int\sum_{y=1}^{K}\left(\frac{p(y|\bm{x})}{p(y\in\mathcal{Y}_{s}|\bm{x})}\ell(\bm{g}(\bm{x}),y)\right)p(\bm{x}|y\in\mathcal{Y}_{s})d\bm{x}\\
&=\left(\sum_{y\in\mathcal{Y}_{s}}\pi_{y}\right)\int\sum_{y=1}^{K}\left(\frac{p(\bm{x}|y\in\mathcal{Y}_{s})}{p(y\in\mathcal{Y}_{s}|\bm{x})}\ell(\bm{g}(\bm{x}),y)\right)p(y|\bm{x})d\bm{x}\\
&=\int\sum_{y=1}^{K}\left(\frac{p(\bm{x},y\in\mathcal{Y}_{s})}{p(y\in\mathcal{Y}_{s}|\bm{x})}\ell(\bm{g}(\bm{x}),y)\right)p(y|\bm{x})d\bm{x}\\
&=\int\sum_{y=1}^{K}\left(p(\bm{x})\ell(\bm{g}(\bm{x}),y)\right)p(y|\bm{x})d\bm{x}\\
&=\int\sum_{y=1}^{K}\ell(\bm{g}(\bm{x}),y)p(\bm{x})p(y|\bm{x})d\bm{x}\\
&=R(\bm{g}).
\end{align*}
\end{proof}
Then we begin to give the formulation of Noise-Robust Sub-Conf learning:
\begin{definition}The risk estimator of Noise-Robust Sub-Conf learning is defined as follow:
\begin{align}
\label{ex}
\hat{\tilde{R}}_{sub}(\bm{g})=\frac{1}{n}\sum_{i=1}^{n}\phi_{i}\sum_{y=1}^{K}\tilde{r}^{y}_{i}\ell(\bm{g}(\bm{x}_{i}),y).
\end{align}
where $\phi(\bm{x})=\frac{p(\bm{x})}{p(\bm{x}|y\in\mathcal{Y}_{\mathrm{s}})}$ and $\phi_{i}=\frac{p(\bm{x}_{i})}{p(\bm{x}_{i}|y\in\mathcal{Y}_{\mathrm{s}})}$.
\end{definition}
Notice that if we substitute $p(\bm{x}|y\in\mathcal{Y}_{\mathrm{s}})$ with $p(\bm{x}|y_{\mathrm{s}})$, (\ref{ex}) is converted to (\ref{ERM-SC}). We can also get the infinite and finite-sample consistency of learning with (\ref{ex}) by conducting such substitution.
\section{Additional Information of Benchmark Experiments}
\subsection{Detailed Information of Benchmark Datasets}
In Section \ref{se}, we used 2 widely-used large-scale benchmark datasets. Here, we report the sources of these datasets and the way we split them.
\begin{itemize}
    \item Fashion-MNIST \cite{F}. It is a 10-class dataset of fashion items. Each instance is a 28*28 grayscale image. Source: \url{https://github.com/zalandoresearch/fashion-mnist}.
    
    \item CIFAR-10 \cite{C}. It is a 10-class dataset for 10 different objects and each instance is a 32*32*3 colored image in RGB format. Source: \url{https://www.cs.toronto.edu/~kriz/cifar.html}.
    
\end{itemize}
In the experiments, we first split 30\% of the dataset for generating confidences. Then we further split 10\% of the dataset for density ratio estimation. Finally we use 10\% of the dataset as the validation set. The rest of the dataset are used as the training set.

\subsection{Detailed Information of the Models and Optimization Algorithm}
For generating confidence scores, the model used for Fashion-MNIST is a 3-layer MLP (d-100-100-100-10) with ReLU and Resnet-18 \cite{ResNet} is used for CIFAR-10.  For density ratio estimation, the output dimension is changed to 1 and other settings are the same. For classification, a 3-layer MLP (d-500-500-500-10) with ReLU is used for Fashion-MNIST, and DenseNet-161 is used for CIFAR-10.

Adam with default momentum was used for optimization in this paper. For generating confidence scores, the epoch number, batch size, learning rate, and weight decay are set to 20, 100, 1e-4, and 1e-4 for Fashion-MNIST, respectively. For density ratio estimation, the epoch number is reduced to 10 and other hyperparameters are the same. When training the classifiers, the epoch number is increased to 100. For CIFAR-10, the learning rate and weight decay are changed to 1e-3 and other hyperparameters are the same as those of Fashion-MNIST. 

\section{Experimental Results on More Benchmark Datasets}
In this section, we show the performance of our methods and baseline methods on more benchmark datasets. We use two widely-used benchmark datasets and report the sources of these datasets. the way we split them is the same as that in the previous section.
\begin{itemize}
\item MNIST \cite{Mnist}. It is a grayscale dataset of handwritten digits from 0 to 9, where the size of the images is 28*28. Source: \url{http://yann.lecun.com/exdb/mnist/}.
\item Kuzushiji-MNIST \cite{Kmnist}. It is a 10-class dataset of cursive Japanese characters ('Kuzushiji'). Source: \url{https://github.com/rois-codh/kmnist}.
\end{itemize}
For generating confidence scores, the model used is a 3-layer MLP (d-100-100-100-10) with ReLU. For density ratio estimation, the output dimension is changed to 1 and other settings are the same. For classification, a 3-layer MLP (d-500-500-500-10) with ReLU is used.

Adam with default momentum was used for optimization in this paper. For generating confidence scores, the epoch number, batch size, learning rate, and weight decay are set to 20, 1000, 1e-4, and 1e-4, respectively. For density ratio estimation, the epoch number is reduced to 10 and other hyperparameters are the same. When training the classifiers, the epoch number is increased to 100. The experimental results are shown in Table \ref{TB3} and \ref{TB4}.

It can be seen from the experimental results that our SC/Sub-Conf learning method outperform other baseline methods in most cases and the NoRSC-Conf learning method can remain effective under extreme noise, which supports our claims in Section \ref{se}.
\begin{table}[t]
\caption{Mean and standard deviation of the classification accuracy over 20 trials for the MNIST dataset. The proposed methods were compared with the baseline Weighted method and fully-supervised method, with different classes used for training. The best and equivalent methods are shown in bold based on the 5\% t-test, excluding fully-supervised method.}
    \label{TB3}
    \centering
\resizebox{0.95\textwidth}{!}{
    \begin{tabular}{c|c|cc|c|c}
    \toprule
    \multicolumn{2}{c|}{Used Classes}&SC/Sub-Conf& NoRSC-Conf&Weighted&Supervised\\
    \midrule
    \multirow{2}*{8}&Accurate&\textbf{89.52$\pm$0.60}&87.87$\pm$1.27&87.68$\pm$1.20&\multirow{10}*{\centering 94.82$\pm$0.53}\\
    
    \cmidrule{2-5} &Noisy&--$\pm$--&\textbf{68.33$\pm$3.62}&63.39$\pm$2.65\\
    \cmidrule{1-5}
    \multirow{2}*{7 \& 9}&Accurate& \textbf{88.46$\pm$0.92}&80.08$\pm$1.86&84.13$\pm$2.00\\
    
    \cmidrule{2-5} &Noisy&--$\pm$--&\textbf{73.48$\pm$1.43}&71.61$\pm$1.78\\
    \cmidrule{1-5}
    \multirow{2}*{0 \& 2 \& 3}&Accurate&\textbf{91.20$\pm$0.26}&88.43$\pm$0.53&90.66$\pm$0.18\\
    
    \cmidrule{2-5} &Noisy&--$\pm$--&\textbf{88.33$\pm$0.15}&\textbf{88.32$\pm$0.47}\\
    \cmidrule{1-5}
    \multirow{2}*{\shortstack{1 \& 5 \& 7 \& 9}}&Accurate&\textbf{91.28$\pm$0.21}&90.89$\pm$0.16&89.72$\pm$0.45\\
    
    \cmidrule{2-5} &Noisy&--$\pm$--&\textbf{86.50$\pm$0.72}&83.85$\pm$1.20\\
    \bottomrule
    \end{tabular}
}
\end{table}
\newpage
\begin{table}[!htbp]
\caption{Mean and standard deviation of the classification accuracy over 20 trials for the Kuzushiji-MNIST dataset. The proposed methods were compared with the baseline Weighted method and fully-supervised method, with different classes used for training. The best and equivalent methods are shown in bold based on the 5\% t-test, excluding fully-supervised method.}
    \label{TB4}
    \centering
\resizebox{0.95\textwidth}{!}{
    \begin{tabular}{c|c|cc|c|c}
    \toprule
    \multicolumn{2}{c|}{Used Classes}&SC/Sub-Conf& NoRSC-Conf&Weighted&Supervised\\
    \midrule
    \multirow{2}*{'Tsu`}&Accurate&\textbf{68.45$\pm$1.42}&67.33$\pm$1.08&64.55$\pm$1.28&\multirow{14}*{\centering 76.06$\pm$0.32}\\
    
    \cmidrule{2-5} &Noisy&--$\pm$--&\textbf{38.14$\pm$2.12}&34.99$\pm$2.27\\
    \cmidrule{1-5}
    \multirow{2}*{'Ha`}&Accurate& \textbf{68.19$\pm$0.56}&67.07$\pm$0.65&66.86$\pm$0.67\\
    
    \cmidrule{2-5} &Noisy&--$\pm$--&\textbf{51.67$\pm$1.26}&48.21$\pm$1.63\\
    \cmidrule{1-5}
    \multirow{2}*{'Ma`}&Accurate&\textbf{66.20$\pm$0.93}&63.67$\pm$1.17&63.56$\pm$0.96\\
    
    \cmidrule{2-5} &Noisy&--$\pm$--&\textbf{38.84$\pm$2.41}&35.80$\pm$2.35\\
    \cmidrule{1-5}
    \multirow{2}*{\shortstack{'Su` \& 'Na`}}&Accurate&\textbf{69.98$\pm$0.72}&67.85$\pm$0.40&68.99$\pm$0.59\\
    
    \cmidrule{2-5} &Noisy&--$\pm$--&\textbf{62.17$\pm$1.26}&61.95$\pm$1.11\\
    \cmidrule{1-5}
    \multirow{2}*{\shortstack{'Ki` \& 'Ya` \& 'Re`}}&Accurate&\textbf{70.23$\pm$0.59}&68.64$\pm$0.62&69.77$\pm$1.58\\
    
    \cmidrule{2-5} &Noisy&--$\pm$--&\textbf{64.60$\pm$1.09}&\textbf{64.47$\pm$0.88}\\
    \cmidrule{1-5}
    \multirow{2}*{\shortstack{‘Ki' \& 'Su` \& 'Ha` \& 'Re`}}&Accurate&\textbf{70.48$\pm$0.47}&\textbf{70.08$\pm$0.51}&68.56$\pm$0.47\\
    
    \cmidrule{2-5} &Noisy&--$\pm$--&\textbf{66.86$\pm$0.74}&65.42$\pm$0.81\\
    \bottomrule
    \end{tabular}
}
\end{table}

\end{document}